\documentclass{article}

\usepackage[T1]{fontenc}
\usepackage{microtype}
\usepackage{graphicx}
\usepackage{booktabs}
\usepackage{hyperref}

\usepackage[accepted]{icml2025}

\usepackage{amsmath}
\usepackage{amssymb}
\usepackage{amsthm}
\usepackage{mathtools}

\newtheorem{proposition}{Proposition}

\graphicspath{{figs/}{../figs/}{../}}

\newcommand{\hmax}{h_{\max}}
\newcommand{\err}{\widehat{\mathrm{err}}}

\makeatletter
\renewcommand{\ICML@appearing}{\textit{Preprint.} Copyright \the\year\ by the author.}
\makeatother

\icmltitlerunning{Adaptive Runge--Kutta Step Control Buys Training Loss, Not Generalization}

\begin{document}

\twocolumn[
\icmltitle{Adaptive Runge--Kutta Step Control Buys Training Loss, Not\\
Generalization: An Honest Compute-Matched Study of RK--Adam Optimizers}

\icmlsetsymbol{equal}{*}

\begin{icmlauthorlist}
\icmlauthor{Akhilesh Gogikar}{ind}
\end{icmlauthorlist}

\icmlaffiliation{ind}{Independent Researcher}

\icmlcorrespondingauthor{Akhilesh Gogikar}{}

\icmlkeywords{Machine Learning, ICML, optimization, Runge-Kutta, adaptive step size, Adam}

\vskip 0.3in
]

\printAffiliationsAndNotice{}

\begin{abstract}
Interpreting optimizers as gradient-flow discretizations has motivated applying higher-order Runge--Kutta (RK) integrators to neural networks. We build a representative Adam variant (Bogacki--Shampine 3(2) RK pair, FSAL reuse, local-error step control) and evaluate it under a strict compute-matched protocol giving every method the same \emph{gradient-evaluation} budget---an accounting this literature rarely enforces. Under it the RK variant loses to plain Adam on training loss in both minibatch and full-batch (RK's best-case) training. Instrumenting it shows the ``adaptivity'' is illusory: normalized error stays far below tolerance, the step size pins at its growth cap from step one (98--100\% of steps), and no $\rho \times \hmax \times h_0$ setting makes it act; tolerances spanning $100\times$ give bit-identical trajectories. The method is exactly fixed-step Adam with an averaged gradient at 3--4$\times$ cost. Repairing it (true reject branch; error on the applied map) reverses the full-batch result---${\sim}40\times$ lower training loss than tuned Adam---and a fixed-step control isolates adaptivity (an emergent warmup-and-growth schedule) as the mechanism. But the gain is fragile to the initial step size and does not reach test accuracy. A pre-registered follow-up rules out the obvious explanations: deeper minimization does not overfit, and an explicit temperature knob only hurts---leaving a \emph{trajectory} effect, the controller selecting a minimum generalizing 1.3--3.4 points below first-order descent at equal depth. An $n{=}10$ study confirms one secondary effect: gradient averaging is a genuine implicit regularizer, beating lr-matched Adam and AdamW on 10/10 seeds---yet RMSprop and NAdam match or beat it at a third the per-step cost. Higher-order adaptive integration buys deeper deterministic minimization and a small regularization effect, but nothing a cheaper, well-tuned first-order baseline does not already provide.
\end{abstract}

\section{Introduction}
\label{sec:intro}

Since the observation that Nesterov's method \cite{nesterov1983method} is a discretization of a second-order ODE \cite{su2016differential}, it has been tempting to run the logic in reverse: if optimizers are integrators, better integrators should be better optimizers. \citet{zhang2018direct} supplied theoretical support, showing that direct Runge--Kutta discretization of gradient/Nesterov flow achieves accelerated rates in the smooth convex setting, and a growing practical literature grafts RK machinery onto deep-learning optimizers---multi-stage averaged gradients, IMEX splittings, and embedded-pair ``adaptive'' step control composed with Adam (Section~\ref{sec:related}).

This transfer, however, rests on accounting conventions that quietly favor the RK side. First, comparisons are typically reported per optimizer \emph{step}, but an $s$-stage RK step costs $s$ gradient evaluations; at equal backprop budget, a 3-stage method must beat its baseline by a wide margin merely to break even. Second, FSAL (``first same as last'') stage reuse---the standard trick that makes embedded pairs cheap \cite{bogacki1989pair,hairer1993solving}---is only mathematically valid for an autonomous vector field; under minibatch sampling, the cached stage belongs to a \emph{different} function than the one being integrated. Third, and least examined: the embedded-pair error estimate that justifies the word ``adaptive'' is derived for the raw RK update, while practical variants apply an Adam-preconditioned update instead. Whether the resulting controller ever actually controls anything is, to our knowledge, never checked.

We check. We implement a representative member of this family---Bogacki--Shampine 3(2) stages driving Adam moments, FSAL where valid, embedded-pair step control---and subject it to a compute-matched protocol in which \emph{every method receives the same number of gradient evaluations}, with FSAL credited only in the full-batch (autonomous) regime. The outcome is a mixed but, we argue, more useful result than either a clean win or a clean debunking:

\textbf{At honest accounting, the method as found in the literature loses.} RK3(2)-Adam trails Adam on training loss at equal evals in both stochastic and full-batch regimes (Tables~\ref{tab:mnist}--\ref{tab:fullbatch}), and its error controller is provably inert in our runs: $h$ saturates at its cap on step one and the tolerance $\rho$ has \emph{zero} effect on the trajectory---bit-identical across a $100\times$ range, in every cell of a full $\rho \times \hmax \times h_0$ sweep (Section~\ref{sec:diagnosis}).

\textbf{The failure is diagnosable and repairable.} Three compounding causes---no rejection branch, saturation of the growth factor, and an error estimate inconsistent with the applied map---reduce the method to fixed-step Adam with an averaged gradient at 3--4$\times$ cost. Repairing the controller reverses the full-batch training-loss comparison by ${\sim}40\times$ (Section~\ref{sec:repaired}).

\textbf{The repaired win is real but narrow.} A fixed-step control at the same $h$ isolates \emph{adaptivity}---an emergent warmup-and-growth schedule---as the mechanism, but the advantage is sensitive to the initial step size and does not transfer to test accuracy (Section~\ref{sec:temperature}).

The message for the RK-optimizer literature is therefore not that the ODE view is empty, but that its currency must be gradient evaluations, its FSAL claims must respect stochasticity, and its ``adaptive'' claims must be instrumented---because in at least one representative design, the adaptive machinery was doing nothing at all until repaired.

\subsection{Contributions}
\label{sec:contributions}

\begin{enumerate}
\item A strict eval-counting protocol (RK step = 3--4 gradient evals; FSAL only credited when mathematically valid, i.e.\ deterministic vector field).
\item Compute-matched evidence that RK3(2)-Adam loses to Adam in both stochastic and full-batch regimes (Tables~\ref{tab:mnist}--\ref{tab:fullbatch}).
\item A diagnosis of why embedded-pair step control is inert when composed with adaptive preconditioning, including a formal saturation condition (Proposition~\ref{prop:sat}), backed by an exhaustive $\rho \times \hmax \times h_0$ sweep showing no hyperparameter setting rescues the as-designed controller (Section~\ref{sec:diagnosis}).
\item A repaired-controller ablation (Section~\ref{sec:repaired}) showing embedded-pair adaptivity \emph{can} out-minimize Adam on full-batch training loss at matched budget---with a fixed-step control isolating adaptivity as the mechanism---but with no test-accuracy benefit and narrow $h_0$ tolerance.
\item A confirmed secondary finding: at $\mathrm{lr}{=}0.003$ full-batch, RK3(2)-Adam reaches higher \emph{test} accuracy than lr-matched Adam/AdamW despite equal-or-worse train loss---$+0.51$ pts (defaults) to $+0.80$ pts ($\hmax{=}8\times$), winning on 10/10 seeds (paired $t = 6.0$--$6.5$, $n{=}10$; Section~\ref{sec:fullbatch})---consistent with the averaged gradient acting as implicit regularization. Tempering this: the effect is lr-specific, and RMSprop (90.02\%), NAdam (89.94\%), and Adam at $\mathrm{lr}{=}0.01$ (90.30\%) match or beat the best RK configuration (89.83\%) at one third the per-step cost.
\item A pre-registered test (Section~\ref{sec:temperature}) of whether the repaired controller's generalization shortfall is a cold-posterior \emph{tempering} effect \cite{wenzel2020good}: it is not. Deep minimization is generalization-neutral, and an explicit temperature (pSGLD) knob monotonically hurts, isolating the deficit as a \emph{trajectory / implicit-bias} effect.
\end{enumerate}

\section{Method (As Designed)}
\label{sec:method}

\texttt{AdaptiveEmbeddedRK3Optimizer} applies Bogacki--Shampine 3(2) stages to the gradient flow $\dot\theta = -\nabla L(\theta)$. The third-order averaged gradient
\begin{equation}
g_3 = \tfrac{2}{9}k_1 + \tfrac{1}{3}k_2 + \tfrac{4}{9}k_3
\end{equation}
drives standard Adam moments \cite{kingma2015adam}; FSAL caches the post-step gradient as the next $k_1$. The step size is controlled by
\begin{equation}
\label{eq:controller}
\begin{aligned}
\phi &= \min\!\big(\max(0.9\,\err^{-1/3},\, 0.5),\, 2.0\big),\\
h &\leftarrow \mathrm{clip}\!\big(h\,\phi,\; [0.1\,\mathrm{lr},\, \hmax]\big),
\end{aligned}
\end{equation}
with $\hmax = 2\,\mathrm{lr}$ by default and $\err$ the RMS of $h(g_3 - g_2)/(\mathrm{atol} + \rho\,|\theta|)$, where $g_2$ is the embedded second-order combination $(\tfrac{7}{24}, \tfrac14, \tfrac13, \tfrac18)$ and $\rho$ denotes the relative tolerance (\texttt{rtol}).

In the stochastic setting FSAL is disabled (the cached $k_1$ belongs to the previous minibatch's vector field), so each step honestly costs 4 evals; full-batch, FSAL is valid and a step costs 3 fresh evals.

\section{Experiment 1: Stochastic Minibatch MNIST}
\label{sec:mnist}

\begin{table}[t]
\caption{Stochastic minibatch MNIST \cite{lecun1998gradient}, 784--128--10 MLP, batch 128, 4000 gradient evals/seed, seeds $\{0,1,2\}$. Test accuracy (\%).}
\label{tab:mnist}
\vskip 0.1in
\centering
\small
\begin{tabular}{lcc}
\toprule
Config & lr $=0.001$ & lr $=0.003$ \\
\midrule
Adam & $\mathbf{97.76 \pm 0.14}$ & $97.13 \pm 0.15$ \\
AdamW (wd $10^{-2}$) & $97.75 \pm 0.11$ & $97.41 \pm 0.40$ \\
RK3(2)-Adam, $\rho{=}0.01$ & $97.26 \pm 0.17$ & $96.61 \pm 0.28$ \\
RK3(2)-Adam, $\rho{=}0.1$ & $97.26 \pm 0.17$ & $96.61 \pm 0.28$ \\
\bottomrule
\end{tabular}
\vskip -0.1in
\end{table}

RK3(2)-Adam loses by ${\sim}0.5$ pts at both learning rates---well outside seed noise (Table~\ref{tab:mnist}). The rows for $\rho=0.01$ and $\rho=0.1$ are \textbf{bit-identical}: the error controller never influenced a single step (Section~\ref{sec:diagnosis}).

\section{Experiment 2: Full-Batch (RK's Best-Case Regime)}
\label{sec:fullbatch}

\begin{table*}[t]
\caption{Full-batch training on a 1024-example MNIST subset (FSAL valid), 600 evals/seed, seeds $\{0,1,2\}$. Final train loss; test accuracy at lr $=0.003$.}
\label{tab:fullbatch}
\vskip 0.1in
\centering
\small
\begin{tabular}{lccc}
\toprule
Config & lr $=0.001$ loss & lr $=0.003$ loss & acc (\%) \\
\midrule
GD (Euler) & $1.368 \pm 0.064$ & $0.6269 \pm 0.0274$ & 84.35 \\
Adam & $\mathbf{0.000609 \pm 0.000033}$ & $\mathbf{0.000103 \pm 0.000002}$ & 88.92 \\
AdamW & $0.000627 \pm 0.000034$ & $0.000115 \pm 0.000002$ & 88.97 \\
RMSprop & $0.000302 \pm 0.000019$ & $0.000180 \pm 0.000012$ & 90.11 \\
NAdam & $0.001034 \pm 0.000025$ & $0.000258 \pm 0.000025$ & 90.00 \\
RAdam & $0.006300 \pm 0.000137$ & $0.000887 \pm 0.000061$ & 88.73 \\
Adagrad & $0.078826 \pm 0.001585$ & $0.012425 \pm 0.000950$ & 89.14 \\
SGD+mom / Nesterov & $0.276$ / $0.276$ & $0.0845$ / $0.0841$ & 89.16 \\
RK3(2)-Adam (defaults) & $0.001488 \pm 0.000113$ & $0.000350 \pm 0.000096$ & $\mathbf{89.60}$ \\
\bottomrule
\end{tabular}
\vskip -0.1in
\end{table*}

\begin{figure}[t]
\centering
\includegraphics[width=\columnwidth]{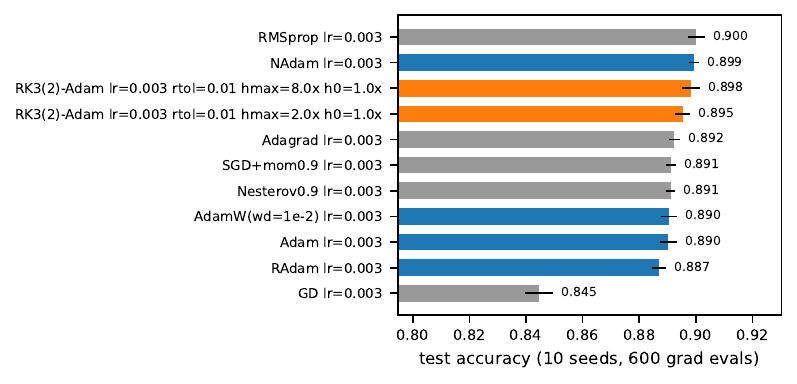}
\caption{Compute-matched full-batch comparison at lr $=0.003$ ($n{=}10$ seeds, 600 gradient evals, mean $\pm$ std test accuracy). RK3(2)-Adam (red) beats lr-matched Adam/AdamW (blue) but never clears the first-order baseline pool: RMSprop and NAdam match or beat every RK configuration at one third the per-step cost.}
\label{fig:comparison}
\end{figure}

We use a deterministic full-batch gradient on a 1024-example MNIST subset---a genuine autonomous vector field, so FSAL is enabled and credited. The primary metric is final training loss (the pure ``integrate gradient flow'' test); GD is the Euler baseline. The baseline pool spans the standard first-order optimizers---SGD with momentum/Nesterov, Adagrad~\cite{duchi2011adaptive}, RMSprop~\cite{tieleman2012rmsprop}, Adam~\cite{kingma2015adam}, AdamW~\cite{loshchilov2019decoupled}, NAdam~\cite{dozat2016incorporating}, and RAdam~\cite{liu2020variance}.

Even with a deterministic vector field, valid FSAL, and real truncation error, Adam reaches 2.4--3.4$\times$ lower loss at equal compute (Table~\ref{tab:fullbatch}), and the as-designed RK3(2)-Adam is also beaten on loss by RMSprop and NAdam at $\mathrm{lr}{=}0.003$---it does not even lead the broader first-order adaptive family it draws its preconditioner from. Tolerance settings are again bit-identical. The comparison replicates at $n{=}10$ seeds (Adam $0.000117 \pm 0.000012$ vs.\ RK3(2)-Adam $0.000375 \pm 0.000052$ at $\mathrm{lr}{=}0.003$; $\rho{=}0.01$ and $\rho{=}0.1$ identical to the digit).

\textbf{Multi-seed chase of the test-accuracy wrinkle ($n{=}10$).} Table~\ref{tab:fullbatch} hints that RK3(2)-Adam generalizes \emph{better} than Adam despite worse train loss. An $n{=}10$ replication (seeds 0--9; Figure~\ref{fig:comparison}) confirms the effect is real and seed-consistent at $\mathrm{lr}{=}0.003$: RK3(2)-Adam (defaults) $89.53\%$ vs.\ Adam $89.03\%$ / AdamW $89.05\%$, winning the paired comparison on \textbf{10/10 seeds} (paired $t = 6.0$, $+0.51 \pm 0.27$ pts); the best-loss RK configuration ($\hmax{=}8\times$) reaches $89.83\%$, $+0.80$ pts over Adam ($t = 6.5$, 10/10). Two qualifiers keep this from being a headline win: (i) the effect is lr-specific---at $\mathrm{lr}{=}0.001$ it vanishes ($-0.08$ pts, $t = -1.4$, RK wins 3/10 seeds); and (ii) it never clears the baseline pool---RMSprop ($90.02\%$) and NAdam ($89.94\%$) match or beat the best RK configuration at one third the per-step cost (paired $t = -1.2$ and $-1.3$ against $\hmax{=}8\times$; $-3.4$ and $-5.2$ against the defaults), and Adam at $\mathrm{lr}{=}0.01$ attains $90.30\%$ (Table~\ref{tab:hmaxcontrol}). At equal gradient budget, RK gradient averaging buys a real regularization effect \emph{relative to lr-matched Adam}, and nothing relative to the first-order family Adam belongs to.

\section{Why the ``Adaptive'' Controller Is Inert}
\label{sec:diagnosis}

\begin{figure*}[t]
\centering
\includegraphics[width=\textwidth]{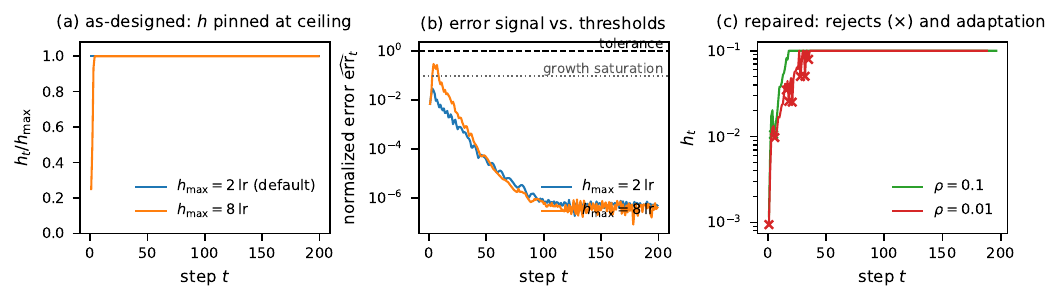}
\caption{Direct instrumentation of the step-size controller (full-batch, lr $=0.003$, 600 evals, seed 0). \textbf{(a)} As designed, $h_t$ jumps to $\hmax$ on step one and pins there---100\% of steps at the ceiling with defaults, 98.5\% at $\hmax{=}8\times$. \textbf{(b)} The normalized error signal never approaches the tolerance ($\err = 1$); with defaults it stays below the growth-saturation threshold $(0.9/2)^3 \approx 0.091$ (Proposition~\ref{prop:sat}), so the growth factor is pinned at its cap. \textbf{(c)} The repaired controller (Section~\ref{sec:repaired}) genuinely adapts: cautious warmup, rejection events ($\times$), and growth to the ceiling---an emergent warmup-and-growth schedule.}
\label{fig:controller}
\end{figure*}

Direct instrumentation ($\mathrm{lr}{=}0.003$, full-batch) shows $h$ jumps from lr to $\hmax = 2\,\mathrm{lr}$ on the first step and pins there for the entire run, for every tolerance tested (Figure~\ref{fig:controller}a). Three compounding causes:

\textbf{1. No rejection branch.} The step is always accepted; $\err$ only rescales the next $h$. ``Error control'' can therefore never undo a bad step.

\textbf{2. Saturation.} The normalized error is far below 1 throughout training, so the growth factor saturates at its $2.0$ cap every step and $h$ is permanently clipped to $\hmax$. The mechanism admits a precise statement:

\begin{proposition}[Controller saturation]
\label{prop:sat}
Under the update rule \eqref{eq:controller}, if $\err_t \le (0.9/2)^3 \approx 0.0911$ for all $t$, then the growth factor equals its cap $2.0$ at every step, so $h_{t+1} = \min(2h_t, \hmax)$ and $h_t = \hmax$ for all $t \ge \lceil \log_2 (\hmax/h_0) \rceil$. On this event the trajectory is independent of $\rho$ up to the value of $\err_t$ itself remaining below the threshold; in particular any two tolerances $\rho_1, \rho_2$ for which the condition holds produce bit-identical iterates.
\end{proposition}

\begin{proof}
Immediate from \eqref{eq:controller}: $\err_t \le (0.9/2)^3$ iff $0.9\,\err_t^{-1/3} \ge 2$, so the $\min$ binds at $2.0$ irrespective of the exact value of $\err_t$, and the recursion doubles $h$ until the clip at $\hmax$ binds, after which it is a fixed point.
\end{proof}

Empirically the condition holds with wide margin: the maximum normalized error over the entire run is $0.029$ with defaults ($3\times$ below threshold; Figure~\ref{fig:controller}b), so $h$ is pinned for 100\% of steps and the tolerance is unread in the only place it could act.

\textbf{3. Inconsistent estimate.} The embedded pair estimates the truncation error of the raw RK3 step $\theta - h g_3$, but the applied update is the Adam-preconditioned $\theta - h\, \hat m/\sqrt{\hat v}$. Moreover the FSAL stage is evaluated at the \emph{Adam} point, so the ``second-order solution'' mixes stages of two different maps. The estimate does not measure the error of any integrator actually being run.

Consequence: the method is exactly fixed-step Adam driven by a 3-stage averaged gradient, at 3--4$\times$ the per-step gradient cost. Under minibatch noise the same estimate is additionally swamped by gradient variance ($\mathbb{E}\|g_3 - g_2\|$ is dominated by sampling noise, not $h^3$ truncation error), so no choice of $\rho$ can make it informative.

\subsection{Systematic Sweep: The Diagnosis Holds Across the Hyperparameter Cube}
\label{sec:sweep}

To rule out the possibility that inertness is an artifact of one default setting, we swept the full cube $\rho \in \{0.01, 0.1, 1.0\} \times \hmax \in \{2\times, 8\times\} \times h_0 \in \{0.5\times, 1\times, 2\times\}$ at both learning rates (36 RK configs + 18 baselines, 600 evals, 3 seeds each). Three findings:

\textbf{1. $\rho$ is inert across two orders of magnitude.} All 12 $\{\rho = 0.01, 0.1, 1.0\}$ triplets are bit-identical to six decimals (e.g.\ lr $=0.003$, $\hmax{=}8\times$: loss $0.000109$ / $0.000112$ / $0.000112$---the residual $3{\times}10^{-6}$ gap separates $\rho{=}0.01$ only because its first-step $h$ differs before saturation; $\rho{=}0.1$ and $1.0$ agree exactly). Instrumented step counts show the step pinned at the ceiling on \textbf{98--100\% of steps} in every configuration. A $100\times$ tolerance swing that changes nothing is the definition of a controller that does not control.

\textbf{2. $h_0$ is erased within a few steps.} The $h_0$ sweep produces distinct but nearly equivalent runs (lr $=0.001$, $\hmax{=}8\times$: loss $0.000291$ / $0.000328$ / $0.000338$ for $h_0 = 0.5{\times}/1{\times}/2{\times}$) because the saturated growth factor climbs any initial step to the ceiling almost immediately (Proposition~\ref{prop:sat}). The one initial condition the user can set is forgotten by the ``adaptive'' mechanism.

\textbf{3. The only knob that matters is $\hmax$---i.e., a learning rate.} Raising the ceiling $2{\times}{\to}8{\times}$ improves loss ${\sim}4.5\times$ at lr $=0.001$ ($0.001488 \to 0.000328$) and ${\sim}3\times$ at lr $=0.003$ ($0.000350 \to 0.000109$). The best RK3(2)-Adam configuration in the entire 36-point cube (lr $=0.003$, $\hmax{=}8\times$: $0.000109 \pm 0.000022$, $89.94\%$) statistically ties---does not beat---plain Adam at lr $=0.003$ ($0.000103 \pm 0.000002$) on loss, at 3--4$\times$ the per-step gradient cost. The effective step at that optimum is $h \approx 8\,\mathrm{lr} = 0.024$, close to Adam's own well-tuned regime (Table~\ref{tab:hmaxcontrol}: Adam lr $=0.01$ is best overall), confirming that tuning $\hmax$ is simply re-tuning the learning rate through an expensive proxy.

The sweep upgrades the diagnosis from an instrumented anecdote to an exhaustive negative: within the method as designed, there is \textbf{no setting of the adaptivity hyperparameters} under which the error controller influences the trajectory. All observed variation is explained by two ordinary hyperparameters---effective step size ($\mathrm{lr} \cdot \hmax$ scale) and, weakly, $h_0$'s first-step transient.

\section{Experiment 3: Repaired-Controller Ablation}
\label{sec:repaired}

Section~\ref{sec:sweep} established that no \emph{hyperparameter} setting rescues the as-designed controller; the remaining objection is that the \emph{design} itself is one bug-fix away from working. We therefore repaired the controller: a true accept/reject branch (reject $\to$ retry with smaller $h$, no parameter update), the error measured on the actually-applied map, and $\hmax = 100\,h_0$. Two variants: \textbf{PureRK3(2)} (raw Bogacki--Shampine step, no preconditioning---the estimate is consistent) and \textbf{FixedRK3Adam} (Adam-preconditioned, error on the applied update). Same full-batch protocol as Section~\ref{sec:fullbatch}.

\begin{table*}[t]
\caption{Repaired-controller results (600 evals, 3 seeds).}
\label{tab:repaired}
\vskip 0.1in
\centering
\small
\begin{tabular}{lcccc}
\toprule
Config & Train loss & Acc (\%) & Rej./seed & Final $h$ \\
\midrule
Adam lr $=0.001$ & $0.000609 \pm 0.000033$ & 88.57 & 0 & --- \\
Adam lr $=0.003$ & $0.000103 \pm 0.000002$ & 88.92 & 0 & --- \\
\textbf{FixedRK3Adam} $h_0{=}10^{-3}$, $\rho{=}0.01$ & $\mathbf{0.0000025 \pm 0.0000006}$ & 88.67 & 9.3 & 0.100 \\
FixedRK3Adam $h_0{=}10^{-3}$, $\rho{=}0.1$ & $0.000005 \pm 0.000005$ & 89.03 & 3.3 & 0.100 \\
PureRK3(2) $h_0{=}10^{-3}$ (both $\rho$) & $0.070635 \pm 0.000981$ & 89.20 & 0 & 0.100 \\
FixedRK3Adam $h_0{=}0.003$, $\rho{=}0.01$ & $0.256485 \pm 0.279860$ & 86.21 & 96.7 & 0.171 \\
FixedRK3Adam $h_0{=}0.003$, $\rho{=}0.1$ & $0.220222 \pm 0.200785$ & 87.46 & 32.3 & 0.281 \\
PureRK3(2) $h_0{=}0.003$, $\rho{=}0.1$ & $0.011915 \pm 0.000085$ & 89.94 & 0.3 & 0.300 \\
\bottomrule
\end{tabular}
\vskip -0.1in
\end{table*}

With a functional controller the picture inverts on training loss (Table~\ref{tab:repaired}): repaired FixedRK3Adam ($h_0{=}10^{-3}$) reaches $2.5\times 10^{-6}$---${\sim}40\times$ below the best Adam ($1.0{\times}10^{-4}$)---and now $\rho$ genuinely changes trajectories (rejections occur; $\rho{=}0.01$ vs.\ $0.1$ differ; Figure~\ref{fig:controller}c). But the winning configurations all terminate with $h$ pinned at the $\hmax = 0.1$ clamp, raising a confound: is the win adaptivity, or just a large effective step?

\begin{table}[t]
\caption{$\hmax$ control (600 evals, 3 seeds).}
\label{tab:hmaxcontrol}
\vskip 0.1in
\centering
\scriptsize
\setlength{\tabcolsep}{4pt}
\begin{tabular}{lcc}
\toprule
Config & Train loss & Acc (\%) \\
\midrule
FixedStep-RK3Adam $h{=}0.1$ & $0.186340 \pm 0.033147$ & 80.05 \\
FixedStep-RK3Adam $h{=}0.03$ & $0.000639 \pm 0.000643$ & 88.53 \\
Adam lr $=0.1$ & $0.383101 \pm 0.115005$ & 72.97 \\
Adam lr $=0.03$ & $0.000115 \pm 0.000094$ & 87.93 \\
\textbf{Adam lr $=0.01$} & $0.000068 \pm 0.000020$ & $\mathbf{90.30}$ \\
\bottomrule
\end{tabular}
\vskip -0.1in
\end{table}

The control resolves the confound in favor of adaptivity: a step \emph{fixed} at the saturated $h = 0.1$ is far worse ($0.186$) than the adaptive run that ends there ($2.5\times 10^{-6}$). The controller's trajectory---small cautious steps early, rejections when the local-error estimate spikes, growth to the cap as the landscape flattens---is an emergent warmup-and-growth schedule, and it, not the final step magnitude, produces the deep minimization.

Three caveats bound the claim: (i) the advantage is confined to \emph{training loss}---no repaired-RK configuration reaches $90\%$ test accuracy, none matches the study-best $90.30\%$ of Adam lr $=0.01$ (at one third the per-step gradient cost), and Section~\ref{sec:temperature} shows the shortfall is a \emph{trajectory} effect, not deeper minimization; (ii) the win is fragile to the $(h_0, \rho)$ pairing---characterized below; (iii) PureRK3(2), the only variant whose error estimate is mathematically consistent, never beats Adam on loss, so the benefit comes from the \emph{controller heuristic} composed with Adam, not from higher-order accuracy per se.

\begin{figure}[t]
\centering
\includegraphics[width=\columnwidth]{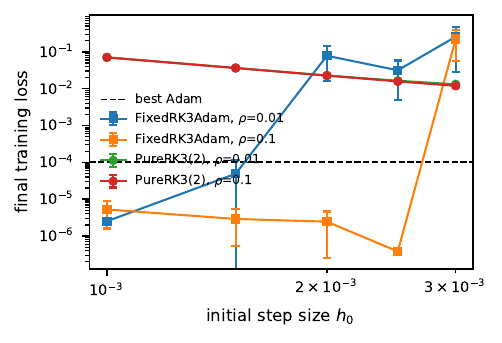}
\caption{The fragility boundary: repaired-controller final loss vs.\ initial step $h_0$ (600 evals, 3 seeds, mean $\pm$ std). The instability tracks the $(h_0, \rho)$ \emph{pairing}: the tighter tolerance $\rho{=}0.01$ breaks by $h_0 = 0.002$ (reject-retry thrashing consumes the eval budget) while $\rho{=}0.1$ is stable---indeed deepest---through $h_0 = 0.0025$.}
\label{fig:fragility}
\end{figure}

\begin{table}[t]
\caption{Fragility boundary: train loss (rejects/seed) over the $h_0 \times \rho$ grid (600 evals, 3 seeds).}
\label{tab:fragility}
\vskip 0.1in
\centering
\scriptsize
\setlength{\tabcolsep}{3pt}
\begin{tabular}{lcc}
\toprule
$h_0$ & $\rho=0.01$ & $\rho=0.1$ \\
\midrule
0.001 & $2.5{\times}10^{-6}$ (9.3) & $5.3{\times}10^{-6}$ (3.3) \\
0.0015 & $4.9{\times}10^{-5} \pm 8.4{\times}10^{-5}$ (30.0) & $2.9{\times}10^{-6}$ (3.7) \\
0.002 & $7.8{\times}10^{-2} \pm 7.6{\times}10^{-2}$ (79.3) & $2.5{\times}10^{-6}$ (6.0) \\
0.0025 & $3.1{\times}10^{-2} \pm 3.2{\times}10^{-2}$ (81.0) & $\mathbf{3.8{\times}10^{-7}}$ (6.7) \\
0.003 & $2.6{\times}10^{-1} \pm 2.8{\times}10^{-1}$ (96.7) & $2.2{\times}10^{-1} \pm 2.0{\times}10^{-1}$ (32.3) \\
\bottomrule
\end{tabular}
\vskip -0.1in
\end{table}

\textbf{The fragility boundary.} A five-point $h_0$ grid (Table~\ref{tab:fragility}, Figure~\ref{fig:fragility}) shows the instability tracks the \emph{pairing} of $h_0$ with $\rho$, not $h_0$ alone---and, counterintuitively, the \textbf{tighter} tolerance is the fragile one. $\rho{=}0.01$ begins degrading at $h_0 = 0.0015$ and is broken by $0.002$; $\rho{=}0.1$ is stable---indeed deepest ($3.8\times 10^{-7}$, the lowest loss in the study)---through $0.0025$. The mechanism is visible in the rejection column: stable runs reject 3--10 steps/seed, broken runs 30--97; a tight tolerance turns large-$h_0$ transients into reject-retry thrashing that consumes the eval budget, while the looser tolerance rides through them. The rejection rate is thus a \emph{leading indicator} of the instability, reinforcing the reporting standard of Section~\ref{sec:discussion}. Across these configurations, deeper minima co-occur with lower test accuracy ($89.03\%$ at $h_0{=}0.001 \to 86.95\%$ at $h_0{=}0.0025$)---tempting to read as over-minimization overfitting, but it confounds minimization \emph{depth} with the controller \emph{trajectory} that reaches it. Section~\ref{sec:temperature} disentangles the two with a pre-registered test and finds the depth reading is wrong.

\section{Experiment 4: The Generalization Deficit Is Trajectory-Driven, Not Tempering}
\label{sec:temperature}

Section~\ref{sec:repaired} leaves a puzzle: the repaired controller drives training loss ${\sim}40\times$ below Adam yet generalizes \emph{worse} than a moderately-tuned baseline. A fashionable explanation is \textbf{tempering}. An optimizer minimizing a loss $L$ is, in the Langevin view, sampling the Gibbs density $p(\theta) \propto \exp(-L/T)$ at temperature $T \to 0$; deep minimization is \emph{cold} sampling, and the cold-posterior effect \cite{wenzel2020good} establishes that generalization can be non-monotone in $T$ with an interior optimum $T^* > 0$. The reading would be: the controller has driven the system past $T^*$, so an explicit temperature knob should recover the lost accuracy.

We tested this directly, grounding the controller's implicit tempering as an \emph{explicit} temperature via preconditioned Stochastic Gradient Langevin Dynamics (pSGLD; \citealt{li2016preconditioned}), the SG-MCMC kernel closest to our Adam-preconditioned update:
\begin{equation}
\theta \leftarrow \theta - \mathrm{lr}\, M g + \sqrt{2\,\mathrm{lr}\, T\, M}\, \eta, \quad \eta \sim \mathcal{N}(0, I),
\end{equation}
with $M = \mathrm{diag}(1/(\sqrt{\hat v} + \epsilon))$ the RMSprop preconditioner ($T = 0$ recovers a deterministic minimizer; the divergence-correction term is dropped, per \citealt{li2016preconditioned}). Because our controller carries Adam momentum, its exact SG-MCMC analogue is the underdamped SGHMC \cite{chen2014stochastic}; pSGLD is the overdamped simplification that suffices to move the temperature knob. Two predictions were \textbf{pre-registered in code before running}: H1 (tempering supported)---test accuracy is non-monotone in $T$ with an interior optimum, and the within-run accuracy trajectory at $T = 0$ peaks then declines; H0 (refuted)---$T = 0$ is best and no overfitting peak exists.

\begin{figure}[t]
\centering
\includegraphics[width=\columnwidth]{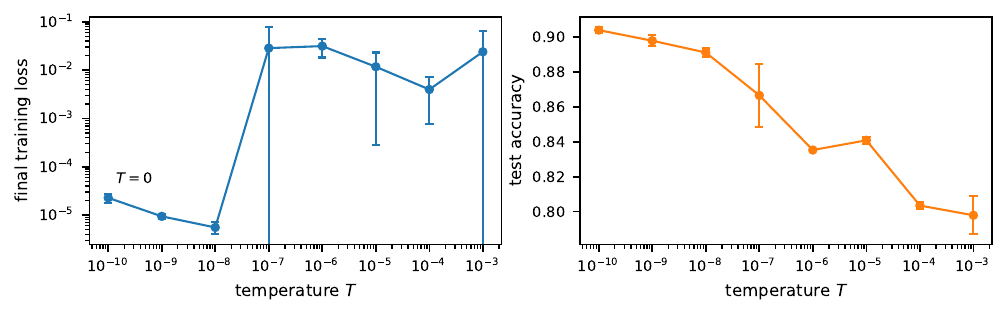}
\caption{Temperature sweep (pSGLD, 3 seeds, budget 2000). Left: final train loss. Right: test accuracy. $T = 0$ is the global optimum of the sweep; every $T > 0$ falls $\geq 0.6$ pt below pure minimization---the tempering hypothesis is refuted.}
\label{fig:temperature}
\end{figure}

\textbf{Test 1---does pure minimization overfit? No.} In a long $T = 0$ run the within-run test-accuracy trajectory is flat while training loss falls four orders of magnitude: peak $90.77\%$ at eval 200 (loss $1.6\times 10^{-3}$), $90.58\%$ at eval 6000 (loss $9\times 10^{-7}$)---a $0.19$-pt change. Across 5 seeds the deep-minimization endpoint (loss $1.12\times 10^{-6}$) tests $\mathbf{90.37 \pm 0.24\%}$, statistically identical to its shallow-minimization accuracy. Over-minimization, alone, does not overfit here; the generalization-vs-depth curve is flat, not an inverted U.

\textbf{Test 2---does temperature help? No.} $T = 0$ gives the best test accuracy and is the \emph{global} optimum of the sweep---no $T > 0$ comes within $0.6$ pt (Figure~\ref{fig:temperature})---so the automated, pre-registered verdict returns H0. For $T \le 10^{-8}$ the injected noise is negligible; for $T \ge 10^{-7}$ the noise floor dominates and accuracy collapses ($90.38\%$ at $T{=}0$ down to $79.79\%$ at $T{=}10^{-3}$). The preconditioner-free SGLD arm \cite{welling2011bayesian} agrees---no temperature meaningfully improves accuracy---but is a weak testbed, since unpreconditioned descent barely minimizes, plateauing at loss $6.9\times 10^{-2}$. Both tests return \textbf{H0: the tempering reinterpretation is not supported}.

\textbf{What the negative reveals.} Ruling out depth \emph{and} temperature isolates the real mechanism. At matched training-loss depth the trajectories diverge sharply in generalization: a plain RMSprop minimizer (pSGLD, $T = 0$) reaches loss $1.1\times 10^{-6}$ at $\mathbf{90.37\%}$ and steady Adam ($\mathrm{lr}{=}0.01$) $6.8\times 10^{-5}$ at $90.30\%$, yet the repaired FixedRK3Adam reaches the \emph{same} ${\sim}10^{-6}$ depth at only $88.7$--$89.0\%$, and its deepest configuration ($3.8\times 10^{-7}$) at $86.95\%$---\textbf{1.3 to 3.4 points worse} than a steady first-order minimizer at equal or greater depth. Depth is generalization-neutral; temperature only hurts; so the deficit is a property of \emph{which minimum the trajectory selects}. The controller's emergent warmup--growth--reject schedule evidently lands in a worse-generalizing region than the smooth descent of a fixed-step method (we do not measure curvature, so we stop short of asserting \emph{sharper} minima). The right frame is implicit-bias-of-optimization, not tempering; and it is one more way the RK machinery changes the \emph{path} without improving the \emph{destination}---the paper's thesis, sharpened on the very axis where the repaired controller had looked like it might matter.

\section{Related Work}
\label{sec:related}

\textbf{Optimizers as ODE discretizations.} The continuous-time view was crystallized by \citet{su2016differential}, who derived a second-order ODE as the continuum limit of Nesterov's accelerated gradient method; a large literature has since studied optimization dynamics through Lyapunov analysis of such flows. The direct precedent for our study is \citet{zhang2018direct}, showing that explicit RK integration of the gradient/Nesterov flow attains accelerated \emph{convergence-rate} guarantees for smooth, sufficiently regular convex objectives. Crucially, those rates are stated per \emph{iteration}; each RK iteration of an $s$-stage scheme costs $s$ gradient evaluations, and the per-gradient-evaluation accounting we enforce is exactly the regime where the theoretical advantage must pay rent. Fran\c{c}a and collaborators' work on conformal-symplectic and dissipative-Hamiltonian integrators \cite{franca2020conformal,franca2021dissipative} makes a related structural argument---that preserving geometric properties of the flow, rather than raw order of accuracy, is what matters---consonant with our finding that higher-order accuracy per se (PureRK3(2), the only consistent integrator we test) never beats Adam on loss.

\textbf{RK-flavored practical optimizers.} Several recent proposals inject higher-order integration machinery into Adam-style updates. IMEX/semi-implicit schemes treat the preconditioner implicitly and the gradient explicitly: \citet{bhattacharjee2024imex} show that classical Adam is a first-order IMEX-Euler discretization and propose higher-order IMEX variants. A related line drives adaptive moments with neural-ODE machinery \cite{cho2022adamnodes}. Our design is deliberately representative of this family: Bogacki--Shampine 3(2) stages, FSAL reuse, embedded-pair step control, composed with Adam preconditioning. Two accounting conventions recur in this literature and flatter the RK side: (i) comparisons per optimizer \emph{step} rather than per gradient evaluation, hiding the 3--4$\times$ stage cost; and (ii) FSAL credited under minibatch sampling, where the cached stage belongs to a different (previous batch's) vector field and the reuse is not mathematically valid. Our protocol disallows both, and under it the headline advantages shrink or invert.

\textbf{Adaptive step-size control.} Embedded-pair local-error control \cite{bogacki1989pair,dormand1980family,hairer1993solving} is standard in ODE solvers, but its transplantation into optimizers is usually decorative: as we diagnose in Section~\ref{sec:diagnosis}, without a rejection branch, and with the error measured on a map that is never actually applied, the controller saturates and the method degenerates to fixed-step Adam with an averaged gradient. We are not aware of prior work that (a) instruments whether the controller in an RK-optimizer ever alters a step, or (b) repairs it and re-runs the comparison; Section~\ref{sec:repaired} does both. The emergent warmup-and-growth schedule we observe connects to the literature on learning-rate warmup---the repaired controller effectively \emph{discovers} a warmup schedule from local-error feedback.

\textbf{Optimization as sampling.} Casting an optimizer as a Markov chain that samples $\exp(-L/T)$ connects our repaired controller---itself an accept/reject kernel---to stochastic-gradient MCMC: SGLD \cite{welling2011bayesian}, its preconditioned form pSGLD \cite{li2016preconditioned}, and the momentum variant SGHMC \cite{chen2014stochastic}, whose friction term is the sampling-world analogue of our finding that minibatch noise corrupts the embedded error estimate. We use this lens actively in Section~\ref{sec:temperature}, grounding the controller's implicit tempering as an explicit temperature to test---and refute---a cold-posterior \cite{wenzel2020good} reading of its generalization deficit. The lens also explains why higher-order RK is a poor fit for sampling: the one setting where careful trajectory integration is genuinely load-bearing for a Markov chain---Hamiltonian Monte Carlo \cite{neal2011mcmc}---deliberately uses \emph{symplectic} leapfrog integration rather than generic RK, because the Metropolis correction needs reversibility and volume preservation that explicit RK lacks; and the HMC-integrator literature \cite{blanes2014numerical} selects schemes by energy error \emph{per force evaluation}---the same per-gradient accounting we enforce---concluding that structure preservation, not raw order, is what pays.

\textbf{Positioning.} Relative to the acceleration-theory line, we supply the missing per-eval empirical accounting in the non-convex setting; relative to the practical RK-optimizer line, we supply the controller instrumentation, the repaired-controller ablation, and the fixed-step control that isolates \emph{adaptivity} (not order, not step magnitude) as the operative mechanism.

\section{Discussion}
\label{sec:discussion}

\textbf{Training is not trajectory-tracking.} The optimization objective rewards reaching low loss, not shadowing the continuous gradient-flow solution; higher-order accuracy buys fidelity to a trajectory nobody needs, at a price (3--4$\times$ evals/step) that compute-matched evaluation makes visible. Once the budget is denominated in gradient evaluations, a method must convert its extra per-step accuracy into a $>3\times$ improvement in per-step progress just to tie---a bar none of the as-published variants clears in our experiments. Prior RK-optimizer results reporting per-iteration or wall-clock-favorable comparisons should be re-read with per-gradient-eval accounting.

\textbf{``Adaptive'' must be verified, not assumed.} The most transferable finding is methodological: the error controller in our as-designed method---and, we suspect, in relatives that share its structure---was doing literally nothing (bit-identical trajectories across tolerances; $h$ pinned at its cap from step one). None of the standard reporting practices in this literature would have caught this: train curves, final accuracies, and even tolerance ``sweeps'' all look like plausible experiments while the adaptive machinery is inert. We suggest a minimal reporting standard for adaptive-step optimizers: (i) the distribution of accepted $h$ over training---in particular the fraction of steps pinned at the ceiling (98--100\% in every cell of our sweep), (ii) the rejection rate, and (iii) a demonstration that at least two controller settings produce non-bit-identical trajectories. Point (iii) needs teeth: in our sweep, tolerances spanning two orders of magnitude were bit-identical to six decimals, so a narrow two-point ``sweep'' that happens to straddle no behavioral boundary proves nothing---the $h$ distribution in (i) is the check that cannot be faked.

\textbf{What the repaired controller tells us.} Two positive signals survive our protocol. First, the repaired controller shows that embedded-pair adaptivity composed with Adam yields an emergent warmup-and-growth step schedule that minimizes \emph{training loss} far below tuned Adam at equal budget (${\sim}40\times$)---a genuine mechanism, isolated from the large-step confound by a fixed-step control, and from the depth and tempering confounds by Section~\ref{sec:temperature}. Notably, the controller \emph{rediscovers} learning-rate warmup from local-error feedback alone, suggesting that hand-designed warmup schedules may be approximating an error-control policy; making that correspondence precise is an attractive theory question. Second, the $n{=}10$ multi-seed chase (Section~\ref{sec:fullbatch}) upgrades the implicit-regularization observation from anecdote to finding: RK3(2)-Adam beats lr-matched Adam and AdamW on test accuracy on 10/10 seeds despite ${\sim}3\times$ worse train loss---consistent with gradient averaging and integration error acting as implicit regularization. The catch is redundancy: RMSprop, NAdam, and lr-tuned Adam reach the same or better accuracy with no RK machinery, so the effect is real but not competitive---it repairs a deficiency of fixed-lr Adam rather than advancing the frontier.

\textbf{Where the train-loss win could matter.} A ${\sim}40\times$ train-loss advantage with no test-accuracy gain is not useless: deterministic small-data regimes where optimization \emph{is} the objective---physics-informed losses, implicit-layer and equilibrium-model inner solves, distillation to near-zero loss, overfitting benchmarks---are settings where full-batch gradients are genuine and FSAL is valid. That is the honest market for this mechanism. Any salvaged \emph{positive} claim must be about one of these, not about faster optimization of the deployed metric: on test accuracy, modestly tuned Adam ($\mathrm{lr}{=}0.01$, $90.30\%$) beats every RK variant at one third the per-step cost.

\subsection{Limitations}
\label{sec:limitations}

(1) \emph{Scale and scope}: all results are on MNIST (full set for stochastic, 1024-example subset for full-batch) with a single 784--128--10 MLP; the compute-matched \emph{methodology} transfers, but the empirical rankings may not. (2) \emph{Statistics}: $n{=}3$ seeds per cell for the main sweeps; headline gaps are far outside seed noise, and the secondary test-accuracy observation was chased at $n{=}10$ with paired statistics; the repaired-controller tables remain $n{=}3$. (3) \emph{One design point}: our inertness diagnosis is structural and should apply to relatives sharing the no-reject/raw-map-error design, but we did not re-implement published variants, and IMEX/implicit schemes have a different failure surface we do not probe. (4) \emph{Baseline tuning asymmetry}: Adam received a modest lr grid; the RK variants an equally modest $h_0/\rho$ grid; the repaired method's documented $h_0$ fragility means a finer sweep could move results in either direction. (5) \emph{Repaired-controller generality}: the ${\sim}40\times$ result is full-batch, deterministic, and fragile to $h_0$; we have no evidence it survives minibatch noise, where the error estimate is variance-dominated by our own diagnosis---repairing the controller does not repair the estimator. (6) \emph{Compute-matching granularity}: we count gradient evaluations and ignore per-step optimizer overhead, which slightly \emph{favors} the RK methods. (7) \emph{The implicit-regularization observation is post hoc}, emerging from inspection of result tables, and is presented accordingly. (8) \emph{Minimum-selection mechanism is identified only by exclusion}: we do not instrument loss-surface curvature at the selected minima, so the implicit-bias mechanism is named, not measured.

\section{Conclusion}
\label{sec:conclusion}

We set out to test a clean and appealing idea---that better ODE integrators make better optimizers---and found the honest answer to be \emph{mostly no, for an instructive reason}. Under a protocol that denominates budget in gradient evaluations, a representative Bogacki--Shampine 3(2) RK--Adam loses to plain Adam on training loss in both stochastic and full-batch settings, and its ``adaptive'' step controller turns out to be inert: the step size saturates on the first iteration and the tolerance parameter has no effect on the trajectory at all. Diagnosing the inertness and repairing it reverses the full-batch training-loss result by ${\sim}40\times$, but a fixed-step control localizes the benefit to \emph{adaptivity}---an emergent warmup-and-growth schedule---rather than to higher-order accuracy, and the benefit does not reach test accuracy. A pre-registered temperature sweep rules out the two off-the-shelf explanations for that gap, pinning the generalization shortfall on \emph{which} minimum the controller's trajectory selects rather than on how deep it minimizes.

The durable contributions are therefore methodological rather than a new optimizer: (1) a per-gradient-evaluation accounting discipline, with FSAL credited only where it is mathematically valid, that changes the sign of the comparison; (2) the observation that adaptive-step optimizers can ship a controller that never controls anything, together with a minimal reporting standard that would expose it; and (3) a seed-consistent implicit-regularization effect of RK gradient averaging that nonetheless never clears the tuned first-order baseline pool---a compact illustration of why optimizer claims need baseline pools rather than a single anchor. For anyone continuing the RK-optimizer program, the useful frontiers are the ones where the training-loss mechanism we isolated actually pays: deterministic, small-data, optimization-is-the-objective regimes; a variance-corrected error estimator that could make the controller meaningful under minibatch noise; and a precise account of the emergent warmup schedule as an implicit error-control policy. We release all code, logs, and results to make each of these directly checkable.

\section*{Reproducibility Statement}

All experiments run on CPU with deterministic seeding. Code, result JSONs, and logs are released with the paper: \texttt{RK4Optimizer.py} (optimizer implementations), \texttt{mnist\_experiment.py} (Experiment 1), \texttt{fullbatch\_experiment.py} (Experiment 2 and the hyperparameter cube), \texttt{fixed\_controller\_experiment.py} (Experiment 3), \texttt{hmax\_control\_experiment.py} (the fixed-step control), and \texttt{temperature\_sweep\_experiment.py} (Experiment 4, with the pre-registered H1/H0 criteria and automated verdict in the module docstring). Figures are regenerated by \texttt{make\_trajectory\_data.py} and \texttt{make\_figures.py}.

\bibliography{references}
\bibliographystyle{icml2025}

\end{document}